\documentclass[journal]{IEEEtran}
\usepackage{mathrsfs}
\usepackage{amsmath,amsthm}
\usepackage{amsfonts}
\usepackage{graphicx,cite,epsfig,amssymb,amsmath}
\usepackage{color,xcolor}
\usepackage{pifont}
\usepackage{stmaryrd}
\usepackage{setspace}
\usepackage{subfigure}
\usepackage{cite}
\usepackage{array}\usepackage{float}\usepackage{multirow}
\usepackage{algorithmic,algorithm}
\usepackage{multirow}
\usepackage{makecell}
\usepackage{verbatim}

\floatstyle{ruled}
\newfloat{algorithm}{tbp}{loa}
\providecommand{\algorithmname}{Algorithm}
\floatname{algorithm}{\protect\algorithmname}

\newtheorem{lemma}{Lemma}

\begin{document}

\title{Semi-asynchronous Hierarchical Federated Learning for Cooperative Intelligent Transportation Systems}

\author{Qimei Chen, Zehua You, and Hao Jiang,
\thanks{Q. Chen, Z. You, and H. Jiang are with the School of Electronic Information, Wuhan University, Wuhan 430072, China (e-mail: \{chenqimei; youzehua; jh\}@whu.edu.cn)}}


\maketitle

\begin{abstract}
\emph{Cooperative Intelligent Transport System} (C-ITS) is a promising network to provide safety, efficiency, sustainability, and comfortable services for automated vehicles and road infrastructures by taking advantages from participants.
However, the components of C-ITS usually generate large amounts of data, which makes it difficult to explore data science.
Currently, federated learning has been proposed as an appealing approach to allow users to cooperatively reap the benefits from trained participants.
Therefore, in this paper, we propose a novel \emph{Semi-asynchronous Hierarchical Federated Learning} (SHFL) framework for C-ITS that enables elastic edge to cloud model aggregation from data sensing.
We further formulate a joint edge node association and resource allocation problem under the proposed SHFL framework to prevent personalities of heterogeneous road vehicles and achieve communication-efficiency.
To deal with our proposed \emph{Mixed integer nonlinear programming} (MINLP) problem, we introduce a distributed \emph{Alternating Direction Method of Multipliers} (ADMM)-\emph{Block Coordinate Update} (BCU) algorithm. With this algorithm, a tradeoff between training accuracy and transmission latency has been derived.
Numerical results demonstrate the advantages of the proposed algorithm in terms of training overhead and model performance.
\end{abstract}


\IEEEpeerreviewmaketitle

\section{Introduction}
With the rapid developments of automated vehicles and Internet-of-Vehicles, \emph{Cooperative Intelligent Transport Systems} (C-ITS) is expected to provide safety, traffic efficiency, and comfortable infotainment services for vehicles with an accurate neighborhood view \cite{C-ITS1}. Specifically, the C-ITS components including vehicles, road infrastructure units, personal devices, and traffic command centers can achieve additional traffic management and convenience with the help of enhanced visions from surrounding participants \cite{C-ITS2}.

Currently, both the \emph{Wireless Access in Vehicular Environments} (WAVE) and the \emph{European Telecommunications Standards Institute} (ETSI) standards have been proposed consecutively to define spectrum allocation, traffic parameters setting, and protocol stack to enable \emph{Quality-of-Service} (QoS) and security procedure for C-ITS \cite{C-ITS-protocol}. Several works have also studied the technologies for C-ITS. 
The authors in \cite{C-ITS-urban} introduce a choreography-based service composition platform to accelerate the reuse-based development for a choreography-based urban coordinative application.
The authors in \cite{a2} proposed a novel topological graph convolutional network to predict the
urban traffic flow and traffic density.

The authors in \cite{C-ITS-QoS} analyze the relationship among security, QoS, and safety for vehicles in C-ITS.  
In addition, the authors in \cite{C-ITS-service} investigate service-oriented cooperation models and mechanisms for autonomous vehicles in C-ITS.
Nevertheless, they have neglected to explore the large amounts of traffic data generated from transportation sensor networks, which is essential for advanced C-ITS applications like connected and autonomous vehicles, traffic control and prediction, and road safety. Although some attempts have been made to utilize data science for C-ITS \cite{C-ITS-data1}, there still exists various stuff challenges to be solved, such as privacy protection, computational complexity, and data heterogeneity.

With the significant advance cloud computing, big data, and machine learning, intelligent transportation is a trend and requirement.
As one of the most dominant distributed machine learning technologies, \emph{Federated Learning} (FL) has been widely studied to deal with the data science recently \cite{FL2,FL3,FL4}. The FL technology allows participant devices to collaboratively build a shard model while preserving privacy data locally \cite{FL5}. Particularly, the prevalent FL algorithm, namely federated averaging, allows each device to train a model locally with its own dataset, and then transmits the model parameters to the central controller for a global aggregation \cite{federated-average}.
Intuitively, FL presents a great potential for C-ITS to facilitate the large data management.
However, directly applying FL to C-ITS still faces three major deficiencies: 1) limited wireless resources; 2) high latency; 3) obliterated data diversity.

According to \cite{FL_MEC_survey,edge-computing}, \emph{Federated Learning training at Edge networks} (FEL) has been regarded as a solution to facilitate the above limitations through bringing model training closer to the data produced locally. Compared with the conventional cloud centric FL approaches, the implementation of FEL can provide high wireless resources utilization since less information is required to be transmitted to the cloud. In addition, FEL has a much lower transmission latency and higher privacy than the conventional FL by making decisions at the edge nodes. In \cite{FEL1}, the authors develop an importance aware joint data selection and resource allocation algorithm to maximize the resource and learning efficiencies. Meanwhile, the authors in \cite{FEL2} propose an adaptive federated learning mechanism in resource constrained edge computing systems. Along the FEL, the authors in \cite{HFEL} propose a novel \emph{Hierarchical Federated Edge Learning} (HFEL) framework, where edge servers deployed with base stations fixedly and can upload edge aggregation model to the cloud. The above HFEL enables great potentials in low latency and energy efficiency.

However, the participant C-ITS devices usually have heterogeneous resources, which lead to \emph{non-independent-identically distributed} (non-iid) private data during the communication \cite{non-IID1,non-IID2}. The existing federated learning methods mainly utilize the synchronous model aggregation mechanism, where the central server needs to wait for the slowest device to complete the training in each communication round \cite{fedavg1,EE-FL}. With heterogeneous data, the transmission latency of each synchronous model aggregation mechanism is unacceptable for time-sensitive devices.
In this way, several works have proposed asynchronous model aggregation methods, where only one participant device would update the global model each time \cite{asyFL1,asyFL2,asyFL3}. Nonetheless, the training round under asynchronous methods is several times than that of synchronous methods. Moreover, due to the asynchrony, gradient staleness may be difficult to control \cite{asyFL3}. Therefore, the authors in \cite{Fed-sensing} design an $n$-softsync aggregation model that can significantly reduce training time by combines the benefits of both synchronous and asynchronous aggregations.

Inspired by the above analyses, we aim to leverage a novel \emph{Semi-asynchronous Hierarchical Federated Learning} (SHFL) framework that can provide safety, efficiency, sustainability, and comfortable services to C-ITS. Specifically, the proposed SHFL framework consists of both edge and cloud layers, where each edge node aggregates all of homogeneous local models and the cloud layer aggregates parts of heterogeneous edge models.  These selected nodes would update the global model once the selected slowest node finishes training, which combines the merits of both synchronous and asynchronous aggregations. For further performance enhancement, we formulate a joint edge node association and resource allocation optimization problem to prevent heterogeneous edge node personalities as well as ensure communication-efficient of the whole system. The objective function is a \emph{Mixed Integer NonLinear Programming} (MINLP) problem, which has been solved by a distributed \emph{Alternating Direction Method of Multipliers} (ADMM)-\emph{Block Coordinate Update} (BCU) algorithm. It is shown that the proposed algorithm can achieve near optimal with low computational complexity.
To the best of our knowledge, this is the first work that applying improved FL into the C-ITS.

Overall, the main contributions of this work can be listed as follows.
\begin{itemize}
  \item We propose a novel SHFL framework by applying the synchronous aggregation model for local-edge and the semi-asynchronous aggregation model for edge-cloud to provide safety, traffic efficiency, and comfortable infotainment services for C-ITS.

  \item To reserve the personalities of heterogeneous edge nodes, we introduce an elastic edge model update method based on the distance between the global model and the edge model.

  \item We formulate a joint edge node association and resource allocation problem to achieve communication-efficiency by achieving a tradeoff between training accuracy and transmission latency. A distributed ADMM-BCU algorithm has been used to solve the MINLP problem.
\end{itemize}

The rest of this paper is organized as follows. Section II introduces the system model and the SHFL learning mechanism. In Section III, we formulate the communication-efficient problem. A joint edge node association and resource allocation strategy is presented in Section IV. Section V presents the numerical results, followed by the conclusions in Section VI.

\section{System Model}
In this work, we aim to design a novel SHFL framework for C-ITS that contains three layers, namely the cloud layer, the edge layer, and the local layer, as shown in Fig. \ref{system model}. Here, we consider the vehicle devices have heterogeneous data structures, namely the local datasets are non-iid.  We let homogeneous devices with similar data size, network bandwidth, and QoS gather in the same edge node. Hence, the edge nodes are heterogeneous. A shared \emph{Deep Neural Network} (DNN) model is distributed over the local devices, which has been trained collaboratively across the devices under their datasets.
Different from conventional FLs, the proposed SHFL framework allows devices train their data locally, homogeneous devices report their computed parameters to the same edge node synchronously, and heterogeneous edge nodes upload their models to the cloud node semi-asynchronously, which can preserve data privacy as well as improve communication efficiency.
In the proposed framework, we assume there has a set of $K$ edge nodes $\mathcal{K}=\{1,...,K\}$.
Any edge node $k$ consists of a set of $N_k$ local devices, denoted as $\mathcal{N}_k=\{L_{k,1},...,L_{k,N_k}\}$.
Under edge node $k\in\mathcal{K}$, local device $n\in\mathcal{N}_k$ owns a local data set $D_{k,n}=\{(\boldsymbol {x}_{j,k,n},{y}_{j,k,n}):j=1,...,|D_{k,n}|\}$, where $\boldsymbol {x}_{j,k,n}$ is the $j$-th input training data sample, ${y}_{j,k,n}$ is the $j$-th corresponding output, and $|D_{k,n}|$ denotes the cardinality of the data set $D_{k,n}$. For simplicity, we assume the SHFL algorithm with a single output. However, this work can be extended to the multiple outputs case.
In what follows, we would introduce each part of the proposed SHFL framework at the $t$-th iteration.
\begin{figure}
  \centering
  \includegraphics[width=0.43\textwidth]{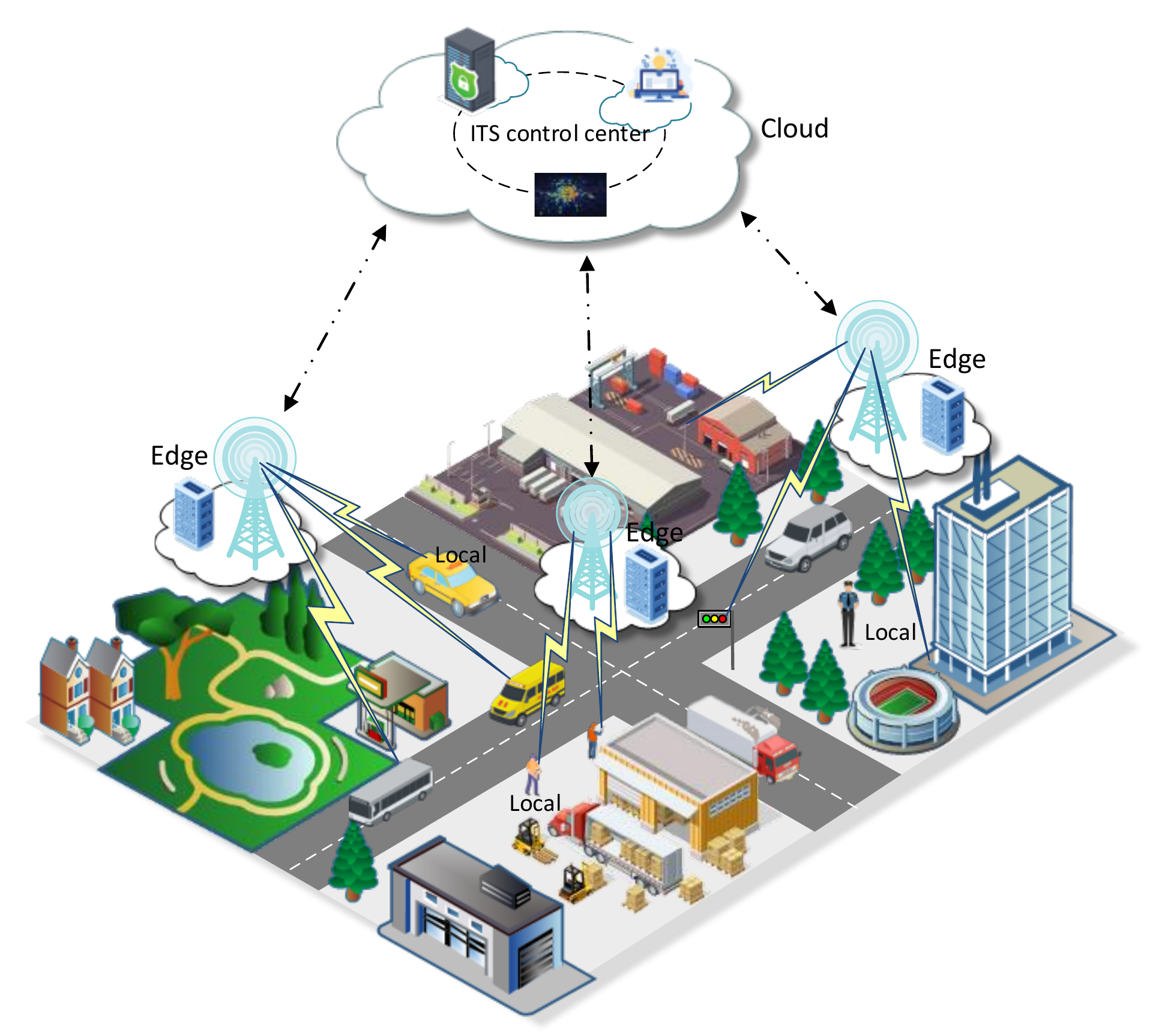}\\
  \caption{Illustration for the SHFL based C-ITS.}\label{system model}
\end{figure}

\subsection{Edge Aggregation}
The edge aggregation stage contains three processes, including local model computation, local model transmission, and edge model aggregation. In detail, local model first trained by local data, then local models respectively transmit to their associated edge nodes for edge aggregation. The detailed processes are as follows.
\subsubsection{Local Model Computation}
Without loss of generality, we consider a supervised machine learning task on device $n\in\mathcal{N}_k$ associated with edge node $k\in\mathcal{K}$, which has a learning model of $\boldsymbol{w}_{k,n}$. We further define $f_{n}(\boldsymbol {x}_{j,k,n}, y_{j,k,n}, \boldsymbol{w}_{k,n})$ as the loss function of data sample $j$ that quantifies the prediction error between data sample $\boldsymbol{x}_{j,k,n}$ and output $y_{j,k,n}$. In this work, we mainly focus on the logistic regression model for the loss function, i.e., $f_{n}(\boldsymbol {x}_{j,k,n}, y_{j,k,n}, \boldsymbol{w}_{k,n})=-\log\left(1+\exp\left(-y_{j,k,n}\boldsymbol {x}_{j,k,n}^{\mathrm{T}}\boldsymbol{w}_{k,n}\right)\right)$. Hence, the loss function of device $n\in\mathcal{N}_k$ associated with edge node $k\in\mathcal{K}$ on dataset $D_{k,n}$ can be defined as
\begin{equation}
\begin{split}
F_{k,n}( \boldsymbol{w}_{k,n})=\frac {1}{|D_{k,n}|}\sum _{j=1}^{|D_{k,n}|} f_{k,n}(\boldsymbol {x}_{j,k,n},y_{j,k,n}, \boldsymbol{w}_{k,n}),
\\~\forall k\in \mathcal{K}, n\in \mathcal{N}_k.
\end{split}
\end{equation}

The local update model of device $n\in\mathcal{N}_k$ in edge node $k\in\mathcal{K}$ can be achieved by
\begin{equation}\label{local_update}
\begin{split}
  \boldsymbol{w} _{k,n}^{t}= \boldsymbol{w} _{k,n}^{t-1}-\eta \nabla F_{k,n}( \boldsymbol{w} _{k,n}^{t-1}),~\forall k\in \mathcal{K}, n\in \mathcal{N}_k,
\end{split}
\end{equation}
where $\eta$ is a predefined learning rate.

Define $C_{k,n}$ as the number of CPU cycles for local device $n\in\mathcal{N}_k$ associated with edge node $k\in\mathcal{K}$ to process one sample data. Assuming each sample data has the same size, the total CPU cycles to run one local iteration is $C_{k,n}|D_{k,n}|$. We further let $f_{k,n}$ be the computation frequency of device $n\in\mathcal{N}_k$ in edge node $k\in\mathcal{K}$. In this way, the related local gradient calculation latency in one round can be formulated as
\begin{equation}
T_{k,n}^{c} = \dfrac {C_{k,n} |D_{k,n}|}{f_{k,n}},~\forall k\in \mathcal{K}, n\in \mathcal{N}_k.
\end{equation}

\subsubsection{Local Model Transmission}
We adopt the \emph{Orthogonal-Frequency-Division Multiple Access} (OFDMA) technique for local uplink transmissions. Define $B_{k,n}$ as the bandwidth allocated to device $n\in\mathcal{N}_k$. Therefore, we have $\sum_{n=1}^{\mathcal{N}_k} B_{k,n} = B_{k}$, where $B_{k}$ is the bandwidth allocated to edge node $k\in\mathcal{K}$ for the transmission between edge node $k\in\mathcal{K}$ and the associated local devices. Meanwhile, we have $\sum_{k=1}^{K}B_{k}\leq B_{e}$, where $B_{e}$ is the total bandwidth allocated for the communication between edge nodes to the local devices.  Therefore, the achievable local uplink data rate from device $n\in \mathcal{N}_k$ to edge node $k\in\mathcal{K}$  can be formulated as
\begin{equation}
 r_{k,n}^{u}= B_{k,n} \log _{2} \left ({1+ \frac{P_{k,n}g_{k,n}}{B_{k,n}N_{0}}}\right)\!,~\forall k\in \mathcal{K}, n\in \mathcal{N}_k,
\end{equation}
where $P_{k,n}$ is the uplink transmission power of device $n\in \mathcal{N}_k$ in edge node $k\in\mathcal{K}$, $g_{k,n}$ denotes the channel gain between local device $n\in \mathcal{N}_k$ and edge node $k\in\mathcal{K}$, and $N_{0}$ means the noise power.

Similarly, the achievable downlink data rate for device $n\in \mathcal{N}_k$ associated with edge node $k\in\mathcal{K}$ can be expressed as
\begin{equation}
r_{k,n}^{d}= B_{k} \log _{2} \left ({1+ \frac{P_{k}g_{k,n}}{B_{k}N_{0}}}\right)\!,~\forall k\in \mathcal{K},  n\in \mathcal{N}_k,
\end{equation}
where $P_{k}$ is the downlink transmission power of edge node $k\in\mathcal{K}$.

In this work, we use the same training model for the whole communication system.
Therefore, the number of model parameters in each level of model transfer has the same size. Denote $Z$ as the data size of the model parameter bits. The local gradient upload latency of device $n\in \mathcal{N}_k$ in edge node $k\in\mathcal{K}$ can be expressed as
\begin{equation}
   T_{k,n}^{u}=\frac {Z}{r_{k,n}^{u}}=\frac {Z}{B_{k,n} \log _{2} \left ({1+ \frac{P_{k,n}g_{k,n}}{B_{k,n}N_{0}}}\right)\!},~\forall k\in \mathcal{K}, n\in \mathcal{N}_k.
\end{equation}

Correspondingly, the edge model download latency of device $n\in \mathcal{N}_k$ in edge node $k\in\mathcal{K}$ can be formulated as
\begin{equation}
T_{k,n}^{d}=\dfrac {Z}{r_{k,n}^{d}}=\dfrac {Z}{B_{k} \log _{2} \left ({1+ \frac{P_{k}g_{k,n}}{B_{k}N_{0}}}\right)\!},~\forall k\in \mathcal{K}, n\in \mathcal{N}_k.
\end{equation}

\subsubsection{Edge Model Aggregation}
In this work, each edge node can receive the updated model parameters from its associated homogeneous devices. Since the devices under one edge node usually have a similar type, we adopt the synchronous aggregation method to average these updated models. It means that the edge node would wait for the slowest node to complete training in each round and collect all the connected devices' updated model parameters. Therefore, the edge model aggregating equation for edge node $k\in\mathcal{K}$ can be formulated as
\begin{equation}\label{edge_aggregation}
 \boldsymbol{w} _{k}^{t}=\frac {\sum _{n=1}^{\mathcal{N}_k} |D_{k,n}| \boldsymbol{w} _{k,n}^{t}}{|D_{k}|},~\forall k\in \mathcal{K},
\end{equation}
where $|D_{k}|=\sum_{n=1}^{\mathcal{N}_k}|D_{k,n}|$ is the total number of data in edge node $k\in \mathcal{K}$.

We omit edge model aggregation time due to its strong computing capability. Hence, the computation and communication latency between each edge $k\in \mathcal{K}$ and the related local devices can be derived as
\begin{equation}
 T_{k}^{edge} = \max_{n\in \mathcal{N}_k}\left(T_{k,n}^{c}+T_{k,n}^{u}+T_{k,n}^{d}\right),~\forall k\in \mathcal{K}.
\end{equation}

\subsection{Cloud Aggregation}
Similarly, the cloud aggregation stage contains two processes, i.e., edge model transmission and cloud model aggregation. Particularly, the selected edge nodes upload their updated model parameters to the cloud for aggregation. The detailed processes are as follows.

\subsubsection{Edge Model Transmission}
Edge nodes would upload their model parameters to the cloud after edge model aggregations.
To ensure uninterrupted transmission from edge to cloud, we also adopt the OFDMA technique.
Hence, the uplink data rate for edge node $k\in\mathcal{K}$ can be expressed as
\begin{equation}
 r_{c,k}^{u}= B_{c,k} \log _{2} \left ({1+ \frac{P_{c,k}g_{c,k}}{B_{c,k}N_{0}}}\right)\!,~\forall k\in \mathcal{K},
\end{equation}
where $B_{c,k}$ is the bandwidth allocated to edge node $k\in\mathcal{K}$ transmits to the cloud node, $P_{c,k}$ is the uplink transmission power of edge node $k\in\mathcal{K}$ to the cloud node, and $g_{c,k}$ denotes the channel gain between edge node $k\in\mathcal{K}$ and the cloud node.

Correspondingly, the downlink data rate from the cloud node to edge node $k\in\mathcal{K}$ can be formulated as
\begin{equation}
r_{c,k}^{d}= B_{c} \log _{2} \left ({1+ \frac{P_{c}g_{c,k}}{B_{c}N_{0}}}\right)\!,~\forall k\in \mathcal{K},
\end{equation}
where $P_{c}$ is the downlink transmission power of the cloud node, $B_{c}$ is the total bandwidth for the transmission between the edge nodes and the cloud. As we would discuss later, only parts of the edge nodes can be selected in each round. Therefore, we have the constraint of
\begin{equation}
\sum_{k=1}^{K}\alpha_{k}B_{c,k}\leq B_{c},
\end{equation}
where $\alpha_{k}\in\{0, 1\}$. Here, $\alpha_{k}=1, \forall k\in\mathcal{K}$ indicates edge node $k$ has been selected, and $\alpha_{k}=0, \forall k\in\mathcal{K}$ otherwise.


In this way, the upload latency from edge node $k\in\mathcal{K}$ to the cloud node can be written as
\begin{equation}
T_{c,k}^{u}=\frac {Z}{r_{c,k}^{u}}=\frac {Z}{B_{c,k} \log _{2} \left ({1+ \frac{P_{c,k}g_{c,k}}{B_{c,k}N_{0}}}\right)\!},~\forall k\in \mathcal{K}.
\end{equation}

Similarly, the downlink latency from the cloud node to edge node $k\in\mathcal{K}$ can be expressed as
\begin{equation}
T_{c,k}^{d}=\dfrac {Z}{r_{c,k}^{d}}=\dfrac {Z}{B_{c} \log _{2} \left ({1+ \frac{P_{c}g_{c,k}}{B_{c}N_{0}}}\right)\!},~\forall k\in \mathcal{K}.
\end{equation}


\subsubsection{Cloud Model Aggregation}
\begin{figure}
  \centering
  \includegraphics[width=0.45\textwidth]{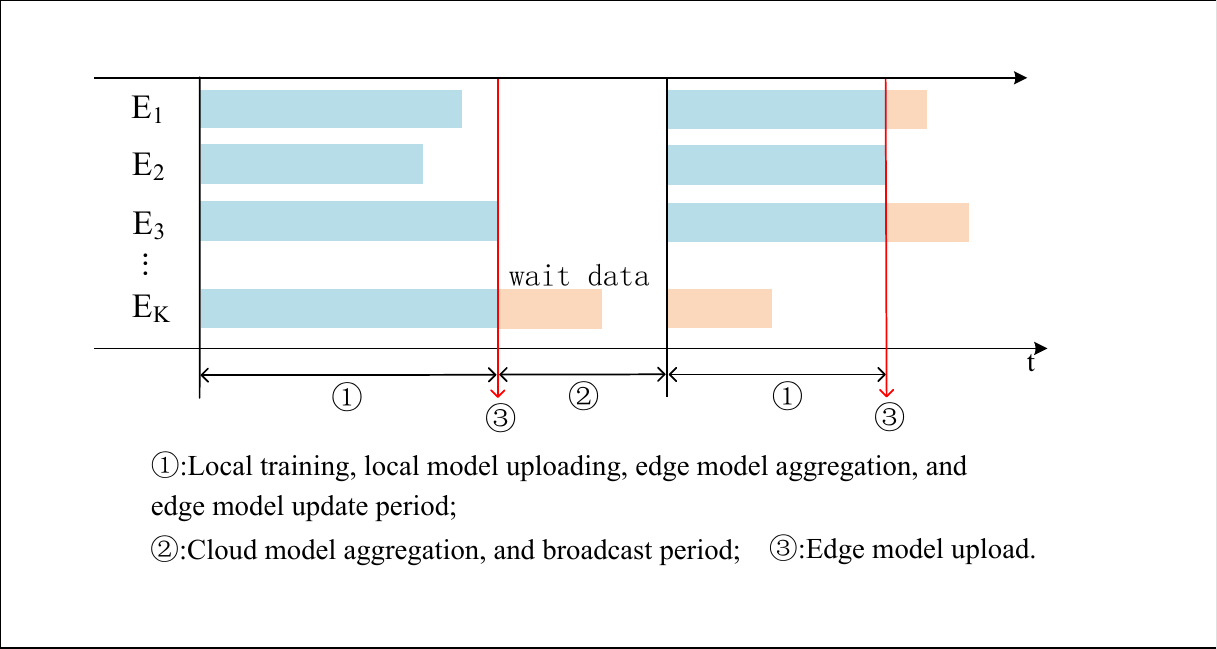}\\
  \caption{Illustration of the semi-asynchronous process.}\label{semi-asynchronous}
\end{figure}
Since these edge nodes correspond to heterogeneous local datasets, their model updated periods various. If we adopt the synchronous aggregation model, the latency for faster training nodes is unacceptable. On the contrary, the asynchronous method has shorter round latency, however, it requires several times of training rounds than the synchronous method. Therefore, in this work, we propose a flexible semi-asynchronous aggregation method by combining the merits of both synchronous and asynchronous methods. As shown in Fig. \ref{semi-asynchronous}, the cloud node would select $|\mathcal{S}^t|=\sum_{k=1}^{K}\alpha_{k}$ edge nodes with the fastest training round for model aggregation, where the set of selected edge nodes is denoted as $\mathcal{S}^t$. Slow nodes would wait for the next communication round to upload their models. Hence, under the semi-asynchronous aggregation method, we can achieve a balance between training accuracy and communication latency. The semi-asynchronous aggregation method can be written as
\begin{equation}\label{cloud_aggregation}
\boldsymbol{w}_{c}^{t}= \boldsymbol{w}_{c}^{t-1}+\sum_{k\in \mathcal{S}^t}\frac{|D_{k}|}{\sum_{k=1}^{K}|D_{k}|}(\boldsymbol{w}_{k}^{t}-\boldsymbol{w}_{c}^{t-1}).
\end{equation}

Also, we ignore the cloud model aggregation latency due to its strong computing capability.
Therefore, the cloud-edge communication latency can be derived as
\begin{equation}
 T_{k}^{cloud} = T_{c,k}^{u}+T_{c,k}^{d},~\forall k\in \mathcal{K}.
\end{equation}

Towards this end, the one-round latency for edge node $k\in\mathcal{K}$ is given by
\begin{equation}
 T_{k} = T_{k}^{edge}+T_{k}^{cloud} ,~\forall k\in \mathcal{K}.
\end{equation}

\subsection{Edge Update Model}
From Eq. \eqref{local_update}, the local updated models are determined by their own characteristics. Since the non-iid devices that connected with one edge node have a similar characteristic, the edge aggregation models are heterogeneous.
Therefore, if we directly use the cloud model to update the edge models, the personalities among edge models would be eliminated. Meanwhile, the accuracy of the cloud model would be decreased. Hence, we introduce a new edge update model based on \cite{weight-update}, which defines a weight distance formula to represent the difference among different weight relatives as
\begin{equation}
dist\left(\boldsymbol{w}^t_{k},\boldsymbol{w}^t_{c}\right) = \frac{\left|\left|\boldsymbol{w}^t_{k}-\boldsymbol{w}^t_{c}\right|\right|}{\left|\left|\boldsymbol{w}^t_{c}\right|\right|}, \forall k\in\mathcal{K}.
\end{equation}
Intuitively, the larger of $dist\left(\boldsymbol{w}_{k},\boldsymbol{w}_{c}\right)$, the greater of the model difference.

Typically, deep learning networks that consist of multiple layers and each layer contains various amounts of weights can be adopted here. For simplicity, we use a small dataset to obtain the layer with the most obvious characteristics, which has been denoted as $\mathcal{L}=\{\ell_1, \ell_2, \cdots\}$.
Thereafter, we introduce a parameter $\varepsilon_k$ to measure the difference between the cloud model and edge model $k$, which can be formulated as
\begin{equation}\label{varepsilon}
\varepsilon_k = \frac{1}{\left|\mathcal{L}\right|}\sum _{\ell\in \mathcal{L}}dist\left(\boldsymbol{w}_{k}^{t,\ell},\boldsymbol{w}_{c}^{t,\ell}\right), \forall k\in\mathcal{K},
\end{equation}
where $\boldsymbol{w}_{k}^{t, \ell}$ and $\boldsymbol{w}_{c}^{t, \ell}$ represent the weight of the $\ell$-th layer of edge model $\boldsymbol{w}^t_{k}$ and cloud model $\boldsymbol{w}^t_{c}$. Meanwhile, $\left|\mathcal{L}\right|$ is the cardinality of $\mathcal{L}$.

From Eq. \eqref{varepsilon}, $\varepsilon_k$ increases with $dist\left(\boldsymbol{w}_{k},\boldsymbol{w}_{c}\right)$.
To keep the personalities, the edge updated model can be derived by
\begin{equation}
\boldsymbol{w}_{k}^{t}\gets \varepsilon_k \boldsymbol{w}_{c}^{t}+(1-\varepsilon_k)\boldsymbol{w}_{k}^{t}, \forall k\in\mathcal{K}.
\end{equation}

\subsection{Learning Procedure of the SHFL Model}
\begin{figure}
  \centering
  \includegraphics[width=0.5\textwidth]{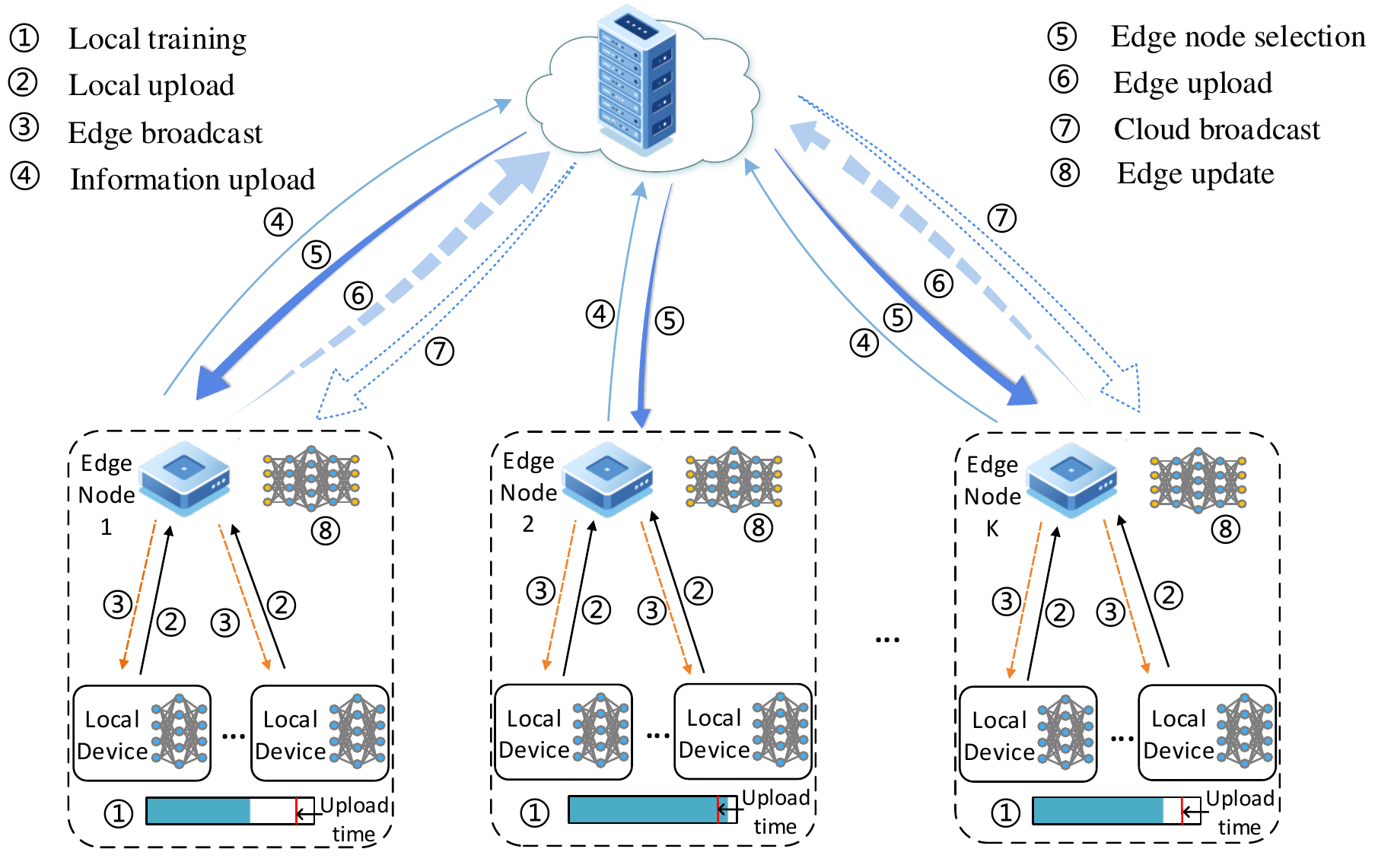}\\
  \caption{Learning Procedure of the proposed SHFL model.}\label{flow chart}
\end{figure}
Based on the definition of SHFL model, the training procedure of the SHFL model at the $t$-th iteration proceeds as follows, which is also shown in Fig. \ref{flow chart}.
\begin{enumerate}
  \item \emph{Local model training and update}: Devices in C-ITS train their learning model and calculate their local gradient as $\nabla F_{k,n}(\boldsymbol{w}_{k,n}^{t}), \forall k\in\mathcal{K}, n\in\mathcal{N}_k$.  After receiving $\boldsymbol{w}_{k}^{t}, \forall k\in\mathcal{S}^t$, devices in the selected edge nodes update their learning model based on Eq. \eqref{local_update}.

  \item \emph{Local model upload}: Local devices upload their concrete models to related edge nodes based on the local-edge bandwidth allocation scheme.

  \item \emph{Edge model aggregation}: After receiving local models, each edge node computes the average edge model based on Eq. \eqref{edge_aggregation}. Since device types among edge nodes are heterogeneous, their edge model aggregations are semi-asynchronous.

  \item \emph{Edge node selection and resource management}: Based on the reports of edge nodes, the cloud node selects a subset of edge nodes $\mathcal{S}^t$ with the fastest training round and determines the bandwidth allocation.

  \item \emph{Selected edge model upload}: The selected edge nodes upload their aggregated models to the cloud node.

  \item \emph{Cloud model aggregation and broadcast}: The cloud server aggregates the uploaded edge models, and then broadcasts the current aggregated model $\boldsymbol{w}_{c}^{t}$ to the selected edge nodes.

  \item \emph{Edge model update and broadcast}: The selected edge servers broadcast the updated model $\boldsymbol{w}_{k}^{t}$ to local related devices.
\end{enumerate}

The procedure starts from $t=1$ and repeats the above steps until convergence.

\section{Problem Formulation}
As discussed earlier, there exists a tradeoff between the training accuracy and the transmission latency. Therefore, our goal in this work is to find a balance between them to provide safety, efficiency, sustainability, and comfortable services for the SHFL based C-ITS framework.

According to Eq. \eqref{local_update}, the local model \emph{Gradient-Norm-Value} (GNV) influences the local model updating, which measures the data importance. The GNV of local device $n\in\mathcal{N}_k$ in edge node $k\in\mathcal{K}$ can be expressed as
\begin{equation}
\begin{split}
{\boldsymbol{g}}_{k,n}^{{\boldsymbol{w}},t} &= \nabla F_{k,n}( \boldsymbol{w} _{k,n}^{t})\\
&=\sum_{D_{k,n}}\frac{\partial f_{k,n}(\boldsymbol {x}_{j,k,n},y_{j,k,n}, \boldsymbol{w}_{k,n}^{t})}{\partial \boldsymbol{w}_{k,n}^{t}}, \forall k\in\mathcal{K}, n\in\mathcal{N}_k.
\end{split}
\end{equation}
With out loss of generality, we leverage the norm of GNV to present the importance, which can be written as
\begin{equation}
\sigma_{k,n}^{t}=\left|\left|{\boldsymbol{g}}_{k,n}^{{\boldsymbol{w}},t}\right|\right|^2, \forall k\in\mathcal{K}, n\in\mathcal{N}_k.
\end{equation}

Since an edge node connects homogeneous local devices, the GNVs among these local devices are approximately equal. Moreover, local devices in one edge node also have similar training duration, hence, all of these training models (GNV) would be uploaded. In this way, the GNV of edge node $k\in\mathcal{K}$ can be defined as
\begin{equation}
\sigma_{k}^{t}=\sum _{n=1}^{\mathcal{N}_k}\sigma_{k,n}^{t}, \forall k\in\mathcal{K}.
\end{equation}
On the contrary, the cloud node associates with heterogeneous edge nodes, the GNVs among them various.
Intuitively, edge nodes with significant gradients have more contributions on model updating and convergence.
Therefore, the cloud would preferentially select impactive edge nodes to upload their information for cloud model aggregation. Then, the GNV of the cloud model can be written as
\begin{equation}
\sigma^{t}=\sum_{k=1}^{K}\alpha^t_{k}\sigma^t_{k}.
\end{equation}
For easy of expression, we remove the iteration $t$ in the following.

Now, we are ready to describe the problem formulation.
The goal of this work is to maximize communication-efficient via joint edge node selection and resource allocation scheduling for a SHFL based C-ITS network. To accelerate the learning process, it is desirable to select more edge nodes with larger data importance. However, to shorten the communication and computation latency, it is better to upload as fewer edge nodes as possible. As a result, the objective function that represents the tradeoff between GNVs and transmission latency can be formulated as
\begin{eqnarray*}\label{Problem}
\hspace{1.1em} \min_{\boldsymbol{\alpha, B_{k,n}, B_{c,k}}}\left(-\rho\sum_{k=1}^{K}\alpha_{k}\sigma_{k}+(1-\rho)\max_{k\in K}\alpha_{k}T_{k}\right),
  \hspace{1.1em} \eqref{mnbs}\label{mnbs}
\end{eqnarray*}
subject to
\begin{subequations}
\begin{equation}\label{cons1}
\sum_{k=1}^{K}\sum_{n=1}^{\mathcal{N}_k}B_{k,n}\leq B_{e},
\end{equation}
\begin{equation}\label{cons2}
\sum_{k=1}^{K}\alpha_{k} B_{c,k}\leq B_{c},
\end{equation}
\begin{equation}\label{cons3}
\alpha_k\in\{0, 1\}, \forall k\in\mathcal{K},
\end{equation}
\end{subequations}
where $B_c=B-B_e$, $B$ is the total bandwidth, $\boldsymbol{\alpha}=[\alpha_1, \alpha_2, \cdots, \alpha_K]^{\mathrm{T}}$, $\boldsymbol{B_{k,n}}=[B_{k,1}, B_{k,2}, \cdots,  B_{k,N_k}]^{\mathrm{T}}$, $\boldsymbol{B_{c,k}}=[B_{c,1}, B_{c,2}, \cdots, B_{c,K}]^{\mathrm{T}}$, and $\rho\in [0,1]$ is the weight factor that controls the tradeoff between data importance and transmission latency.

Obviously, \eqref{Problem} is a MINLP problem, which is NP-hard. In the following, we would introduce an ADMM-BCU method to find the joint edge node selection and resource allocation strategy.

\section{Joint Edge Node Selection and Resource Allocation}
As known to us, all of the steps in the learning procedure are independent with the optimal scheduling decision. Denote $T_{k}^{'}=\max_{n\in \mathcal{N}_k}(T_{k,n}^{c}+T_{k,n}^{u})+T_{c,k}^{u}$, the original problem \eqref{Problem} can be rewritten as
\begin{equation}\label{P2}
\min_{\boldsymbol{\alpha, B_{k,n}, B_{c,k}}}\max_{k\in K}~\left\{-\rho\sum_{k=1}^{K}\alpha_{k}\sigma_{k}+(1-\rho)\alpha_{k}T_{k}^{'}\right\},
\end{equation}
subject to \eqref{cons1}, \eqref{cons2}, and \eqref{cons3}.

To solve the above min-max problem, we first denote $X=\max_{n\in \mathcal{N}_k}\left(T_{k,n}^{c}+T_{k,n}^{u}\right)$ and
$Y=\max_{k\in K}(-\rho\sum_{k=1}^{K}\alpha_{k}\sigma_{k}+(1-\rho)\alpha_{k}(X+T_{c,k}^{u}))$. By applying the parametric method \cite{parametric-algorithm},  \eqref{P2} can be transformed into
\begin{equation*}
\min_{{\boldsymbol{\alpha,B_{k,n},B_{c,k}}},X,Y} Y, \tag{\ref{P1}}
\end{equation*}
\begin{subequations}\label{P1}
subject to \eqref{cons1}, \eqref{cons2}, \eqref{cons3}, and
\begin{equation}\label{consX}
T_{k,n}^{c}+T_{k,n}^{u} \leq X,
\end{equation}
\begin{equation}\label{consY}
 -\rho\sum_{k=1}^{K}\alpha_{k}\sigma_{k}+(1-\rho)\alpha_{k}(X+T_{c,k}^{u}) \leq Y.
\end{equation}
\end{subequations}

Nevertheless, \eqref{P1} is still a unsolvable MINLP problem. We then introduce an auxiliary variable $\boldsymbol{\tilde{\alpha}}$ to deal with the binary vector $\boldsymbol{\alpha}$ \cite{binary}. Thereafter, \eqref{P1} can be reformulated as
\begin{equation*}
\min_{{\boldsymbol{\alpha,B_{k,n},B_{c,k}}},X,Y} Y, \tag{\ref{P3}}
\end{equation*}
\begin{subequations}\label{P3}
subject to \eqref{cons1}, \eqref{cons2}, \eqref{cons3}, \eqref{consX}, \eqref{consY}, and
\begin{equation}\label{alpha1}
\boldsymbol{\alpha} - \boldsymbol{\tilde{\alpha}} = \boldsymbol{0},
\end{equation}
\begin{equation}\label{alpha2}
\alpha_k (1-\tilde{\alpha}_{k}) = 0,
\end{equation}
\begin{equation}\label{alpha}
0 \leq \alpha_{k} \leq 1.
\end{equation}
\end{subequations}

Problem \eqref{P2} is transferred as a convex problem with equality constraints now. Hereinafter, we introduce the Augmented Lagrangian (AL) method to solve problem \eqref{P3} through penalizing and dualizing the equality constraints \eqref{alpha1} and \eqref{alpha2} as
\begin{multline}\label{AL}
\min_{\boldsymbol{\alpha},\boldsymbol{\tilde{\alpha}},\boldsymbol{B_{k,n}},\boldsymbol{B_{c,k}},X,Y}Y+ \frac{1}{2\nu}\sum_{k=1}^{K}\left[\alpha_k \left(1-\tilde{\alpha}_{k}\right)+\nu\lambda_{k}\right]^2 \\+ \frac{1}{2\nu}\sum_{k=1}^{K}\left(\alpha_k-\tilde{\alpha}_{k}+\nu\tilde{\lambda}_{k}\right)^2,
\end{multline}
subject to \eqref{cons1}, \eqref{cons2}, \eqref{consX}, \eqref{consY}, and \eqref{alpha}.

Here, $\nu$ is the non-negative penalty parameter,
and $\boldsymbol{\lambda} = [\lambda_{1},\lambda_{2},...,\lambda_{K}]$ and $\boldsymbol{\widetilde{\lambda}} = [\widetilde{\lambda}_{1},\widetilde{\lambda}_{2},...,\widetilde{\lambda}_{K}]$ denote the dual variable vectors correspond to constraints \eqref{alpha1} and \eqref{alpha2}, respectively.

Nevertheless, \eqref{AL} is still a coupled problem due to the multiply variables of $Y$, $\boldsymbol{B_{k,n}}$, and $\boldsymbol{B_{c,k}}$. Therefore, in this work, we propose a distributed ADMM-BCU algorithm that can iteratively approach a near optimal stable solution with low computational complexity. Specifically, during each iteration, \eqref{AL} is decomposed into edge node selection and resource allocation subproblems, which aim to solve the blocks of $\{\boldsymbol{\alpha}, X, Y\}$ and $\{\boldsymbol{\tilde{\alpha}}, \boldsymbol{B_{k,n}}, \boldsymbol{B_{c,k}}\}$, respectively.

\paragraph{The optimal edge node selection $\{\boldsymbol{\alpha}, X, Y\}$}

Under the fixed resource allocations block $\left\{\boldsymbol{\widetilde{\alpha}}, \boldsymbol{B_{k,n}}, \boldsymbol{B_{c,k}}\right\}$, the edge node selection optimization subproblem over the variable block $\{\boldsymbol{\alpha}, X, Y\}$ can be rewritten as
\begin{multline}\label{PX}
\min_{\boldsymbol{\alpha},X,Y} Y
+ \frac{1}{2\nu}\sum_{k=1}^{K}\left[\alpha_k \left(1-\tilde{\alpha}_{k}\right)+\nu\lambda_{k}\right]^2 \\
+ \frac{1}{2\nu}\sum_{k=1}^{K}\left(\alpha_k-\tilde{\alpha}_{k}+\nu\tilde{\lambda}_{k}\right)^2,
\end{multline}
subject to \eqref{cons2}, \eqref{consX}, \eqref{consY}, and \eqref{alpha}.

Obviously, \eqref{PX} is a convex problem, which can be solved by standard tools, such as CVX.

In what follows, we provide the closed form expression of the optimal edge node selection by introducing Lemma \ref{thm1}.
\begin{lemma}\label{thm1}
The optimum edge node selection $\boldsymbol{\alpha^*}$ can be expressed as
\begin{equation}\label{sol_alpha}
\alpha_{k}^{*}=\frac{\nu M_{k}+\nu\kappa_{k}\left(\rho\sigma_{k}-\left(1-\rho\right)\left(X+T_{c,k}^{u}\right)\right)}{\left(1-\tilde{\alpha}_{k}\right)^{2}+1}, \forall k\in\mathcal{K},
\end{equation}
where $M_{k}=\psi_{k}-\xi_{k}-\lambda_{k}\left(1-\tilde{\alpha}_{k}\right)-\tilde{\lambda}_{k}-\phi_{k}B_{c,k}+\frac{1}{\nu}\tilde{\alpha}_{k}$.
$\boldsymbol{\phi}$, $\boldsymbol{\zeta}$, $\boldsymbol{\kappa}$, $\boldsymbol{\psi}$, and $\boldsymbol{\xi}$ are Lagrangian multipliers correspond to constraints \eqref{cons2}, \eqref{consX}, \eqref{consY}, and \eqref{alpha}, which can be found by one-dimensional search methods based on the complementary slackness conditions.
\end{lemma}
\begin{proof}
To find the optimal $\alpha_k, \forall k\in\mathcal{K}$, we apply for the Lagrangian dual method, which can rearrange Eq. \eqref{PX} with respect to $\alpha_k, \forall k\in\mathcal{K}$ as
\begin{equation}
\small{
\begin{split}
&L(\boldsymbol{\alpha_{k}},\boldsymbol{\phi_{k}},\boldsymbol{\zeta_{k}},\boldsymbol{\kappa_{k}},\boldsymbol{\psi_{k}})\\
&=\frac{Y}{\sum_{k=1}^{K}\sigma_{k}\alpha_{k}}+ \frac{1}{2\nu}\sum_{k=1}^{K}\left[\alpha_k \left(1-\tilde{\alpha}_{k}\right)+\nu\lambda_{k}\right]^2 \\&+ \frac{1}{2\nu}\sum_{k=1}^{K}\left(\alpha_k-\tilde{\alpha}_{k}+\nu\tilde{\lambda}_{k}\right)^2\\
&+\phi_{k}\left(\sum_{k=1}^{K}\alpha_{k}B_{c,k}-B_{c}\right)-\psi_{k}\alpha_{k}-\xi_{k}\left(1-\alpha_{k}\right)\\
&+\zeta_{k}\left(T_{k,n}^{c}+T_{k,n}^{u}-X\right)\\
&+\kappa_{k}\left(-\rho\sum_{k=1}^{K}\alpha_{k}\sigma_{k}+\left(1-\rho\right)\alpha_{k}\left(X+T_{c,k}^{u}\right)-Y\right).
\end{split}}
\end{equation}

Calculate the first-order partial derivatives with respect to $\alpha_k, \forall k\in\mathcal{K}$, we derive that
\begin{equation}
\begin{split}
&\frac{\partial L(\alpha_{k},\phi_{k},\zeta_{k},\kappa_{k},\psi_{k})}{\partial\alpha_{k}}\\&=\frac{1}{\nu}\left[\alpha_{k}\left(1-\tilde{\alpha}_{k}\right)^{2}+\alpha_{k}-\tilde{\alpha}_{k}\right]
+\lambda_{k}\left(1-\tilde{\alpha}_{k}\right)\\&+\tilde{\lambda}_{k}+\phi_{k}B_{c,k}-\psi_{k}+\xi_{k}\\
&+\kappa_{k}\left(-\rho\sigma_{k}+(1-\rho)\left[X+T_{c,k}^{u}\right]\right), \forall k\in\mathcal{K}.
\end{split}
\end{equation}

Setting $\frac{\partial L(\alpha_{k},\mu_{k},\phi_{k},\zeta_{k})}{\partial\alpha_{k}}=0$, we have
\begin{equation}\label{sol_alpha}
\alpha_{k}^{*}=\frac{\nu M_{k}+\nu\kappa_{k}\left(\rho\sigma_{k}-(1-\rho)(X+T_{c,k}^{u})\right)}{\left(1-\tilde{\alpha}_{k}\right)^{2}+1},
\end{equation}
where $M_{k}=\psi_{k}-\xi_{k}-\lambda_{k}\left(1-\tilde{\alpha}_{k}\right)-\tilde{\lambda}_{k}-\phi_{k}B_{c,k}+\frac{1}{\nu}\tilde{\alpha}_{k}$.  This ends the proof.
\end{proof}

From Lemma \ref{thm1}, we find that the edge node selection is mainly determined by the edge node importance $\sigma_{k}$ and the uplink transmission latency from edge $k$ to the cloud $T^u_{c,k}$. Intuitively, the cloud preferentially selects the edge node with either a larger edge node importance or a smaller uplink transmission latency that can improve the communication-efficiency.

\paragraph{The optimal bandwidth allocation $\left\{\boldsymbol{\widetilde{\alpha}}, \boldsymbol{B_{k,n}}, \boldsymbol{B_{c,k}}\right\}$}

Similarly, under the fixed edge node selection block $\left\{\boldsymbol{\alpha}, X, Y\right\}$, the resource allocation optimization subproblem over the block $\left\{\boldsymbol{\widetilde{\alpha}}, \boldsymbol{B_{k,n}}, \boldsymbol{B_{c,k}}\right\}$ can be rearranged as
\begin{multline*}
\min_{\boldsymbol{\widetilde{\alpha}},\boldsymbol{B_{k,n}},\boldsymbol{B_{c,k}}} \frac{1}{2\nu}\sum_{k=1}^{K}\left[\alpha_k \left(1-\tilde{\alpha}_{k}\right)+\nu\lambda_{k}\right]^2 \\+ \frac{1}{2\nu}\sum_{k=1}^{K}\left(\alpha_k-\tilde{\alpha}_{k}+\nu\tilde{\lambda}_{k}\right)^2, \tag{\ref{PB}}
\end{multline*}
\begin{subequations}\label{PB}
subject to \eqref{cons1}, \eqref{cons2}, and
\begin{equation}\label{ConstraintX_PB}
\frac{Z}{B_{k,n}\log_{2}\left(1+\frac{P_{k,n}g_{k,n}}{B_{k,n}N_{0}}\right)}+\dfrac{C_{k,n}N_{k,n}}{f_{k,n}} \leq X,
\end{equation}
\begin{multline}\label{ConstraintY_PB}
-\rho\sum_{k=1}^{K}\alpha_{k}\sigma_{k}+(1-\rho)\alpha_{k}\cdot \\
\left(X+\frac{Z}{B_{c,k}\log_{2}\left(1+\frac{P_{c,k}g_{c,k}}{B_{c,k}N_{0}}\right)}\right) \leq Y.
\end{multline}
\end{subequations}

Also, it is easy to observe that \eqref{PB} is a convex problem. For ease of analyses, we write this problem under the Lagrangian dual formulation, where \eqref{PB} can be rearranged as

\begin{equation}
\small{
\begin{split}
&L\left(\boldsymbol{\tilde{\alpha}_{k}},\boldsymbol{B_{k,n}},\boldsymbol{B_{c,k}},\boldsymbol{\beta_{k}},\boldsymbol{\varkappa_{k}},\boldsymbol{\varphi_{k}},\boldsymbol{\tau_{k}}\right)\\
&=Y+ \frac{1}{2\nu}\sum_{k=1}^{K}\left[\alpha_k \left(1-\tilde{\alpha}_{k}\right)+\nu\lambda_{k}\right]^2 \\
&+ \frac{1}{2\nu}\sum_{k=1}^{K}\left(\alpha_k-\tilde{\alpha}_{k}+\nu\tilde{\lambda}_{k}\right)^2\\
&+\beta_{k}\left(\sum_{k=1}^{K}\sum_{n=1}^{\mathcal{N}_k}B_{k,n}-B_{e}\right)+\varkappa_{k}\left(\sum_{k=1}^{K}\alpha_{k}B_{c,k}-B_{c}\right)\\
&+\varphi_{k}\left[\frac{Z}{B_{k,n}\log_{2}\left(1+\frac{P_{k,n}g_{k,n}}{B_{k,n}N_{0}}\right)}+\dfrac{C_{k,n}N_{k,n}}{f_{k,n}}-X\right]\\
&+\tau_{k}\left(-\rho\sum_{k=1}^{K}\alpha_{k}\sigma_{k}\right.\\
&\left.
+(1-\rho)\alpha_{k}\left[X+\frac{Z}{B_{c,k}\log_{2}\left(1+\frac{P_{c,k}g_{c,k}}{B_{c,k}N_{0}}\right)}\right]-Y\right),
\end{split}}
\end{equation}
where $\boldsymbol{\beta}$, $\boldsymbol{\varkappa}$, $\boldsymbol{\varphi}$, and $\boldsymbol{\tau}$ are the Lagrangian multipliers corresponding to constraints \eqref{cons1}, \eqref{cons2}, \eqref{ConstraintX_PB}, and \eqref{ConstraintY_PB}, respectively.

By taking $\frac{\partial L(\tilde{\alpha}_{k},B_{k,n},B_{c,k},\beta_{k},\varkappa_{k},\varphi_{k},\tau_{k})}{\partial B_{k,n}} = 0$, the optimal local-edge uplink bandwidth allocation $B_{k,n}^{*}$ can be derived by
\begin{multline}\label{B_kn}
\varphi_{k}\frac{Z}{\left(B_{k,n}^{*}\log_{2}\left(1+\frac{P_{k,n}g_{k,n}}{B_{k,n}^{*}N_{0}}\right)\right)^{2}}\left[\log_{2}\left(1+\frac{P_{k,n}g_{k,n}}{B_{k,n}^{*}N_{0}}\right)-\right.\\\left.\frac{P_{k,n}g_{k,n}}{\left(P_{k,n}g_{k,n}+B_{k,n}^{*}N_{0}\right)\ln2}\right]=\beta_{k}, ~\forall k\in \mathcal{K}, n \in \mathcal{N}_k.
\end{multline}
From Eq. \eqref{B_kn}, the optimal local-edge bandwidth allocation $B_{k,n}^{*}$ is mainly influenced by the related channel conditions $\frac{P_{k,n}g_{k,n}}{N_{0}}$.

\begin{algorithm}[H]
\caption{Joint edge node selection and resource allocation strategy.}
{\normalsize
\begin{algorithmic}[1]
\STATE Initialize the gradient norm value $\sigma_{k}, \forall k\in\mathcal{K}$.

\STATE Set the minimum successive divergence threshold of the objective function $\epsilon^{\min}$ and the maximum iteration number $R^{\mathrm{max}}$.

\STATE Set the iteration number $\jmath = 0$.

\STATE Initialize the auxiliary variables $\boldsymbol{\tilde\alpha^{(\jmath)}}$, $\boldsymbol{\lambda^{(\jmath)}}$, $\boldsymbol{\tilde\lambda^{(\jmath)}}$.

\STATE \textbf{While} $\epsilon^{(\jmath)} \geq \epsilon^{\min}$ and $\jmath\leq R^{\mathrm{max}}$ \textbf{do}

\STATE \hspace{3ex} Calculate the optimal device selection decision $\alpha_{k}^{*(\jmath)} $according to \eqref{sol_alpha}.
		
\STATE \hspace{3ex} Calculate the optimal bandwidth $B_{k,n}^{*(\jmath)}, B_{c,k}^{*(\jmath)}$ according to \eqref{B_kn} and \eqref{B_ck}.
		
\STATE \hspace{3ex} Obtain $\tilde{\alpha}_{k}^{*(\jmath)}$ according to \eqref{Sol_dualAlpha}.

\STATE \hspace{3ex} Update $\boldsymbol{\lambda^{(\jmath)}}$ and $\boldsymbol{\tilde\lambda^{(\jmath)}}$ according to
\begin{equation}\label{lambda}
\boldsymbol{\lambda}^{(\jmath)} = \boldsymbol{\lambda}^{(\jmath-1)} + \frac{1}{\nu}\boldsymbol{\alpha}^{*(\jmath-1)}\left(1-\boldsymbol{\tilde{\alpha}}^{*(\jmath-1)}\right),
\end{equation}
\begin{equation}\label{dualLambda}
\boldsymbol{\tilde{\lambda}}^{(\jmath)} = \boldsymbol{\tilde{\lambda}}^{(\jmath-1)} + \frac{1}{\nu}\left(\boldsymbol{\alpha}^{*(\jmath-1)}-\boldsymbol{\tilde{\alpha}}^{*(\jmath-1)}\right).
\end{equation}

\STATE \hspace{3ex} Update $\epsilon^{(\jmath)}$ according to \eqref{sigma}.

\STATE \hspace{3ex} Set $\jmath=\jmath+1$.

\STATE \textbf{End while}
\end{algorithmic}}\label{tableA}
\end{algorithm}

Alternatively, by taking $\frac{\partial L(\tilde{\alpha}_{k},B_{k,n},B_{c,k},\beta_{k},\varkappa_{k},\varphi_{k},\tau_{k})}{\partial B_{c,k}} = 0$, the optimal edge-cloud uplink bandwidth allocation $B_{c,k}^{*}$ can be obtained by
\begin{multline}\label{B_ck}
\tau_{k}\frac{\left(1-\rho\right)Z}{\left(B_{c,k}^{*}\log_{2}\left(1+\frac{P_{c,k}g_{c,k}}{B_{c,k}^{*}N_{0}}\right)\right)^{2}}\left[\log_{2}\left(1+\frac{P_{c,k}g_{c,k}}{B_{c,k}^{*}N_{0}}\right)-\right.\\\left.\frac{P_{c,k}g_{c,k}}{\left(P_{c,k}g_{c,k}+B_{c,k}^{*}N_{0}\right)\ln2}\right]=\varkappa_{k}, ~\forall k\in \mathcal{K}, n \in \mathcal{N}_k.
\end{multline}
Obviously, the optimal edge-cloud uplink bandwidth allocation $B_{c,k}^{*}$ has a similar rule with $B_{k,n}^{*}$.

Thereafter, by setting $\frac{\partial L(\tilde{\alpha}_{k},B_{k,n},B_{c,k},\beta_{k},\varkappa_{k},\varphi_{k},\tau_{k})}{\tilde{\alpha}_{k}} = 0$, we can obtain the optimal auxiliary variable $\tilde{\alpha}^*_{k}$ as

\begin{equation}\label{Sol_dualAlpha}
\tilde{\alpha}_{k}^{*}=\frac{\alpha_{k}\left(1+\alpha_{k}+\nu\lambda_{k}\right)+\nu\tilde{\lambda}_{k}}{1+\alpha_{k}^{2}},~ \forall k \in \mathcal{K}.
\end{equation}

The detailed procedure for the joint edge node selection and resource allocation scheduling is presented in Algorithm \ref{tableA}.

In Algorithm \ref{tableA}, $\epsilon^{(\jmath)}$ means the successive divergence of the objective function at the $\jmath$-th iteration \cite{binary}, which can be defined as
\begin{equation}\label{sigma}
\epsilon^{(\jmath)}=F^{(\jmath)}-F^{(\jmath-1)},
\end{equation}
where
$F= Y
+\frac{1}{2\nu}\sum_{k=1}^{K}\left[\alpha_k \left(1-\tilde{\alpha}_{k}\right)+\nu\lambda_{k}\right]^2
+\frac{1}{2\nu}\sum_{k=1}^{K}\left(\alpha_k-\tilde{\alpha}_{k}+\nu\tilde{\lambda}_{k}\right)^2.$

\section{Numerical Results}
In this section, we conduct experiments to evaluate the theoretical analyses and test the performance of the proposed algorithm.
\subsection{Experiment Settings}
\emph{CNN model settings:} For exposition, we consider the learning task of training image classifiers, which are implemented on a \emph{Convolutional Neural Network} (CNN) model, namely VGGNet $16$ \cite{CNN}. The corresponding training dataset is CIFAR-10, which contains $50000$ training images and $10000$ testing images with $10$ categories. To simulate the distributions of heterogeneous data based mobile devices, all data samples are first sorted by digital labels, and then divided into $100$ shards of size $500$ and each local device is assigned with $5$ shards.
The batch size of each local device is set as $50$ and the average quantitative bit number of each parameter is set as $16$ bits. In addition, we adopt the \emph{Stochastic Gradient Descent} (SGD) optimizer, and the learning rate for the CNN model is set as $0.1$. The computation frequency of each local device is randomly set between $2$ GHz to $4$ GHz.

\emph{Wireless communication settings:} We consider a hierarchical SHFL communication network consists of one cloud node and $10$ edge nodes. Each edge node connects with two local devices. Both edge nodes and local devices  are uniformly distributed under the coverage of the cloud node. The total bandwidth is set as $20$ MHz. Moreover, the uplink transmission powers of each local device and edge node are set as $10$ dBm and $24$ dBm, respectively. Also, the downlink transmission powers of each edge node and the cloud node are set as $10$ dBm and $24$ dBm, respectively.
Furthermore, we utilize the transmission pass loss model of $128.1 + 37.6 \log(d[\mathrm{km}])$. Meanwhile, the noise power spectral density is set as $N_0=-174$ dBm/Hz.

In the ADMM-BCU algorithm, we set the non-negative penalty parameter $\nu$ as $1$. The minimum successive divergence threshold $\epsilon^{min}$ is set as $10^{-4}$. In addition, the maximum iteration number of ADMM-BCU algorithm is set as $200$.

\subsection{SHFL Performance}
In this subsection, we present the convergence performance of the proposed SHFL model.
We first introduce the following baselines.
\begin{itemize}
  \item \emph{Random selection}: Under this circumstance, CNN is implemented with random data selection, where both $5$ and $8$ edge nodes randomly selective conditions are respectively considered.
  \item \emph{Full selection}: Under this circumstance, CNN is implemented by selecting all of the edge nodes.
  \item \emph{Normal edge update}: Edge nodes directly use the broadcast cloud model as their updated model.
\end{itemize}
\begin{figure}
\center
\subfigure[Training accuracy.]{
\begin{minipage}[b]{0.42\textwidth}
\includegraphics[width=\textwidth]{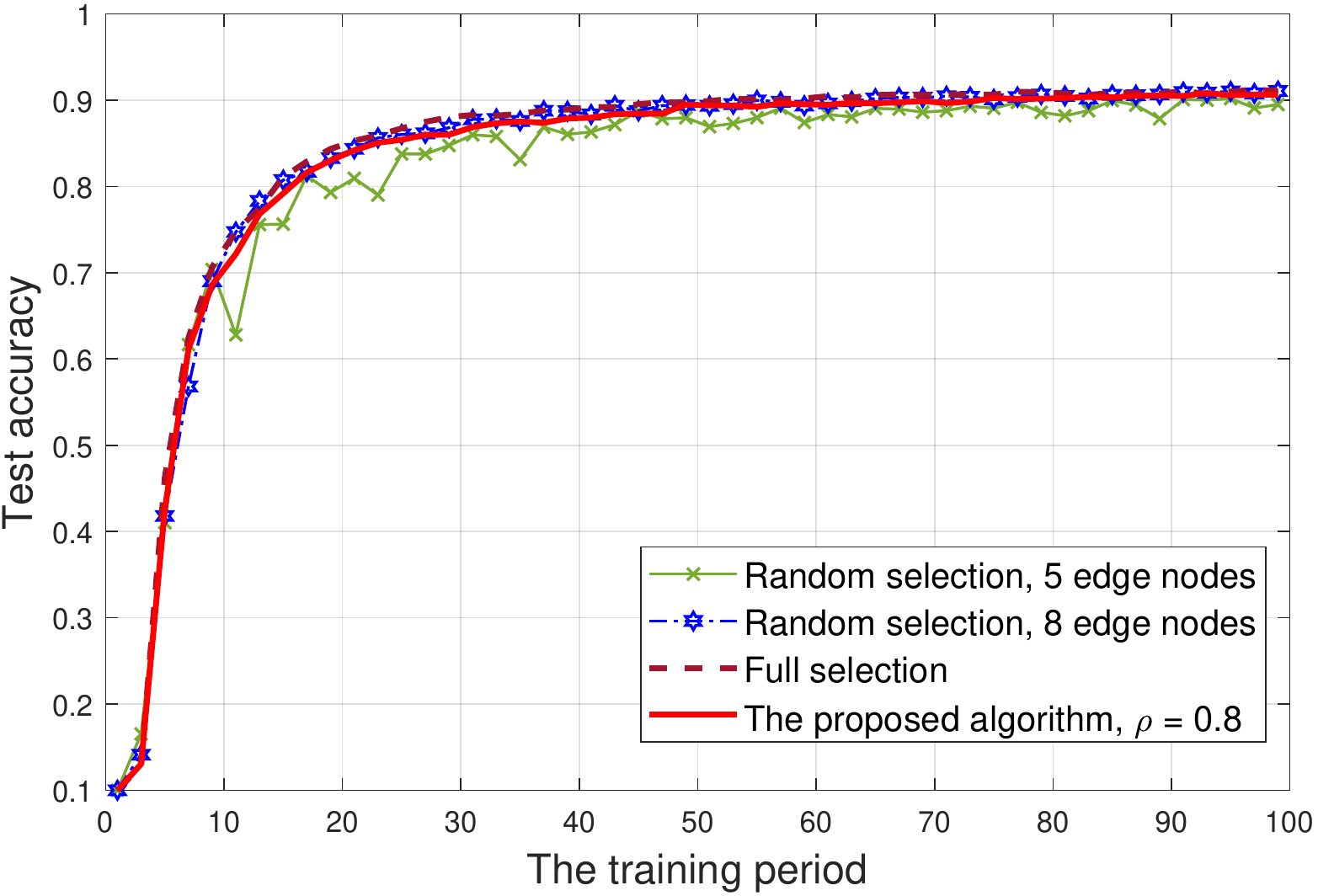}
\end{minipage}\label{conv_semi}
}
\subfigure[Training loss.]{
\begin{minipage}[b]{0.42\textwidth}
\includegraphics[width=\textwidth]{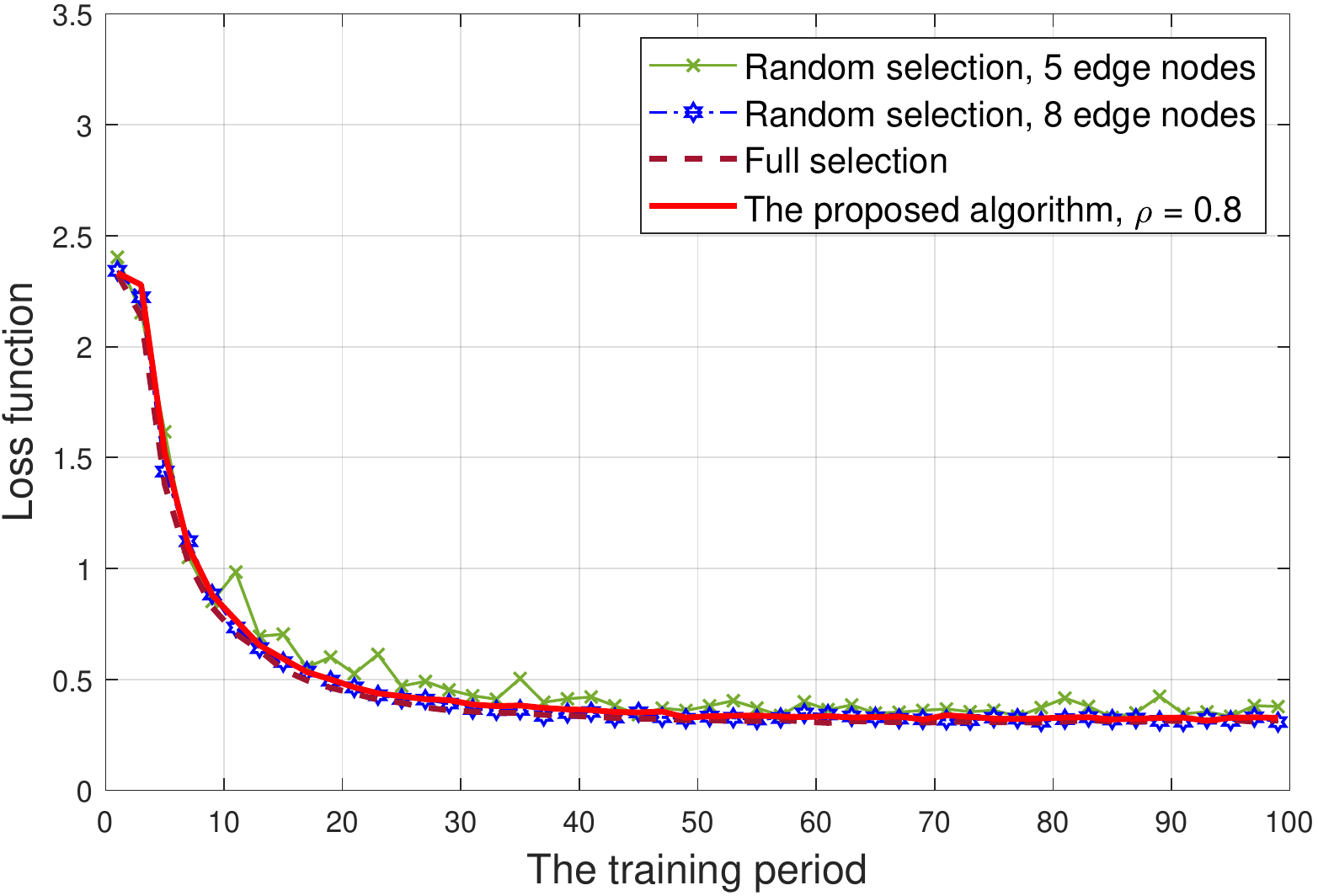}
\end{minipage}\label{conv_loss}
}
\caption{Convergence performance of the proposed CNN model under different algorithms.}\label{accuracy}
\end{figure}
For simplicity, we assume the transmissions from the selected edge nodes to the cloud node are uniformly allocated, totally $5$ MHz. Meanwhile, the transmission bandwidths from local devices to edge nodes are also set as the same, totally $15$ MHz. Moreover, we set the weighted factor $\rho$ under the proposed algorithm as $0.8$.
Fig. \ref{accuracy} shows the convergence performance of the proposed CNN based SHFL model. From this figure, we can find that the VGG-16 network starts to converge at about $70$ communication rounds for both the random selection scheme with $8$ edge nodes, the Full selection scheme, and the proposed scheme. However, the random selection scheme with $5$ edge nodes presents the worst convergence performance. Intuitively, it is because the more devices to be selected, the larger data information can be provided to the neural network, and thus faster convergence.  Moreover, due to the non-iid datasets, each node has different contributions. Therefore, the random selection scheme may play a side effect on the whole model, leading to a decreasing model accuracy.
Overall, the proposed algorithm shows a near to the full selection scheme convergence and accuracy, which can achieve better
performance than the baselines that would be discussed latter.

\begin{figure}
\center
\subfigure[Training accuracy.]{
\begin{minipage}[b]{0.42\textwidth}
\includegraphics[width=\textwidth]{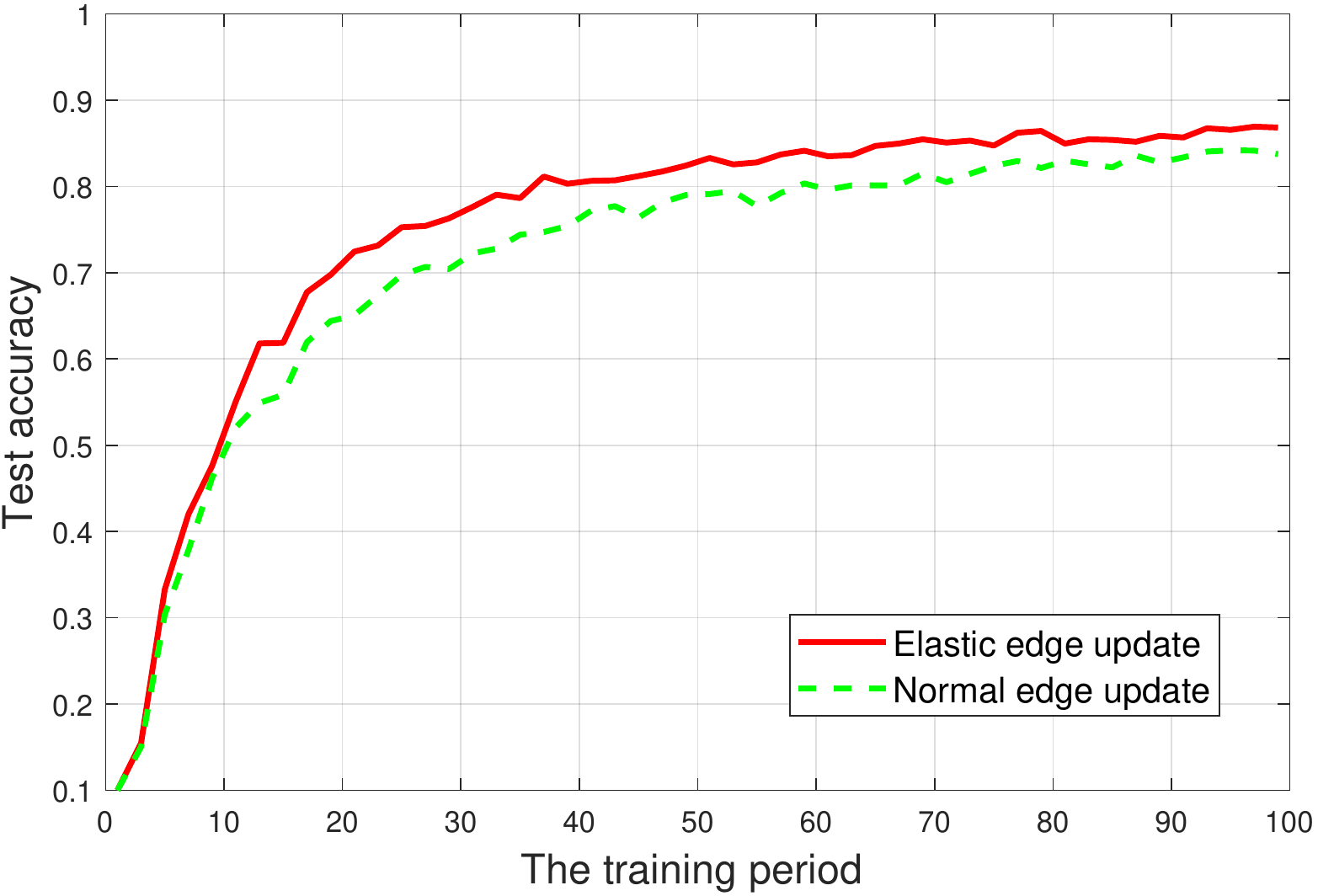}
\end{minipage}\label{conv_semi}
}
\subfigure[Training loss.]{
\begin{minipage}[b]{0.42\textwidth}
\includegraphics[width=\textwidth]{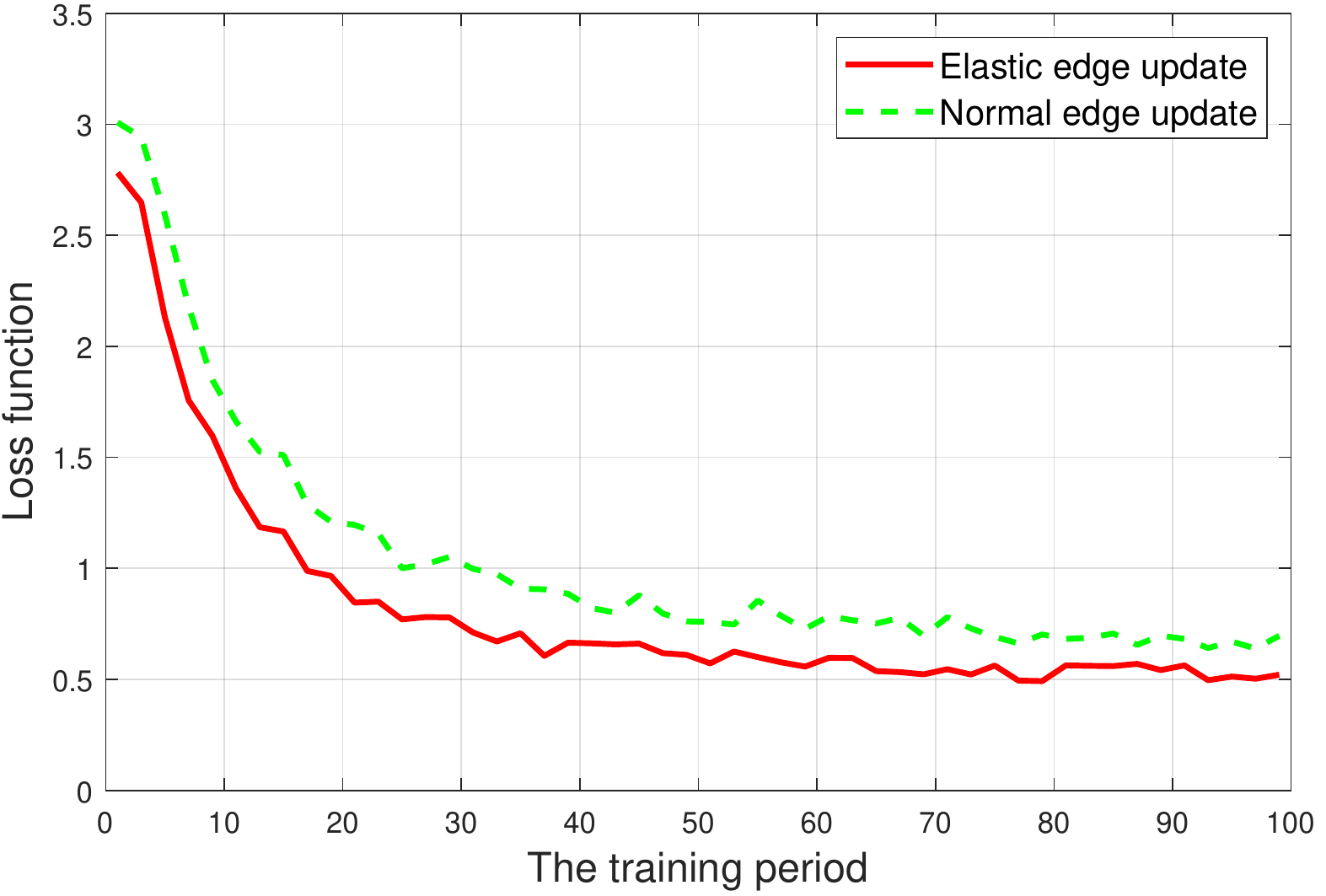}
\end{minipage}\label{conv_BCD}
}
\caption{Performance influence from the edge update model.}\label{edge_update}
\end{figure}
Fig. \ref{edge_update} presents the performance influence from the edge update model.  From this figure, we find that either the training accuracy or the training loss under the elastic edge update model is better than that of the normal edge update model. The fluctuation of these curves are mainly due to the non-iid data form. Therefore, we can conclude that the elastic edge update model is significant to keep the personalities of the edge nodes.

\subsection{The Scheduling Performance}
In this subsection, we mainly verify the scheduling performance of the proposed algorithm.
\begin{figure}
  \centering
  \includegraphics[width=0.42\textwidth]{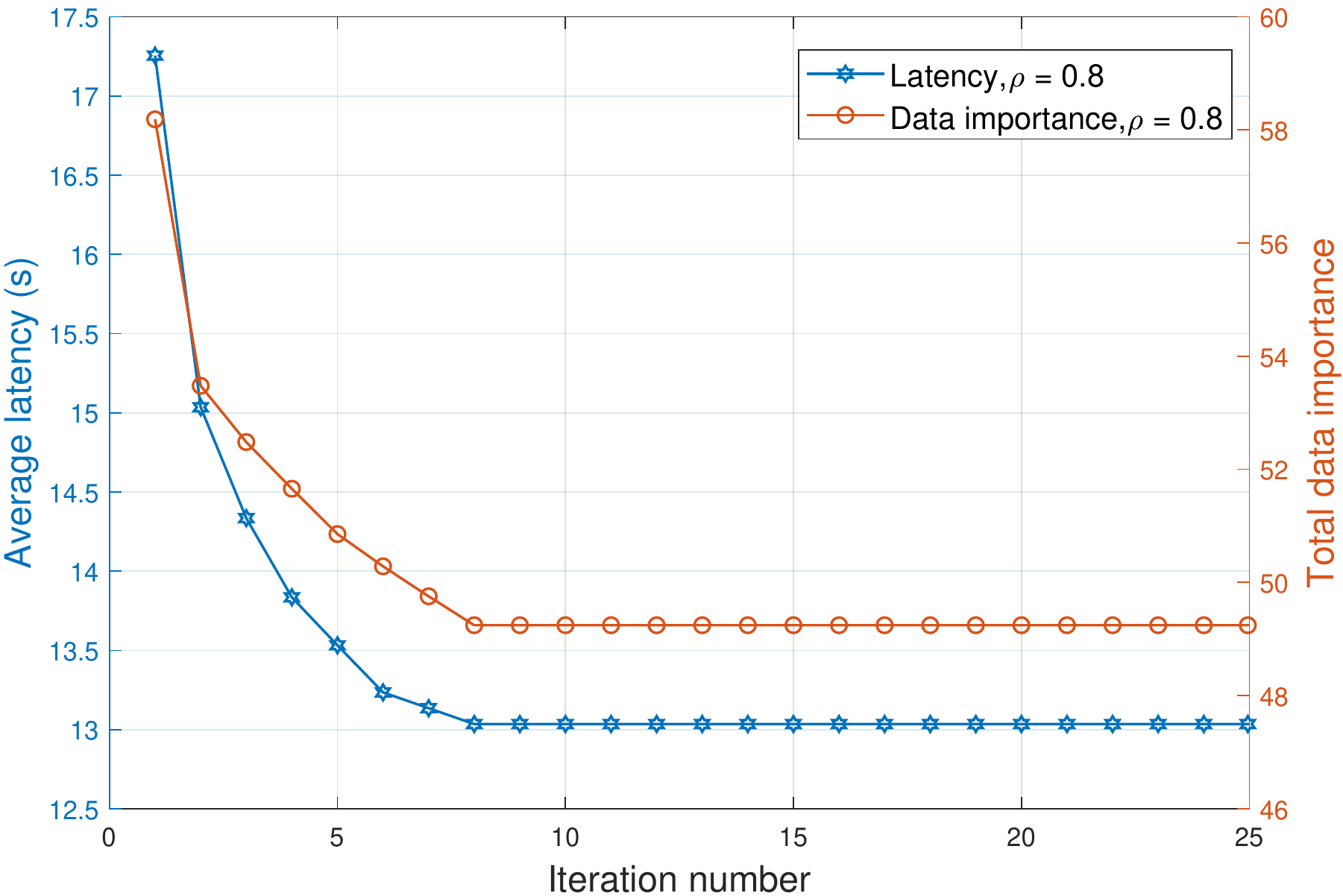}\\
  \caption{Convergence performance of the proposed ADMM-BCU algorithm.}\label{convergence}
\end{figure}
In Fig. \ref{convergence}, we shows that the proposed ADMM-BCU algorithm has a fast convergence and a low computational complexity.

\begin{figure}
  \centering
  \includegraphics[width=0.42\textwidth]{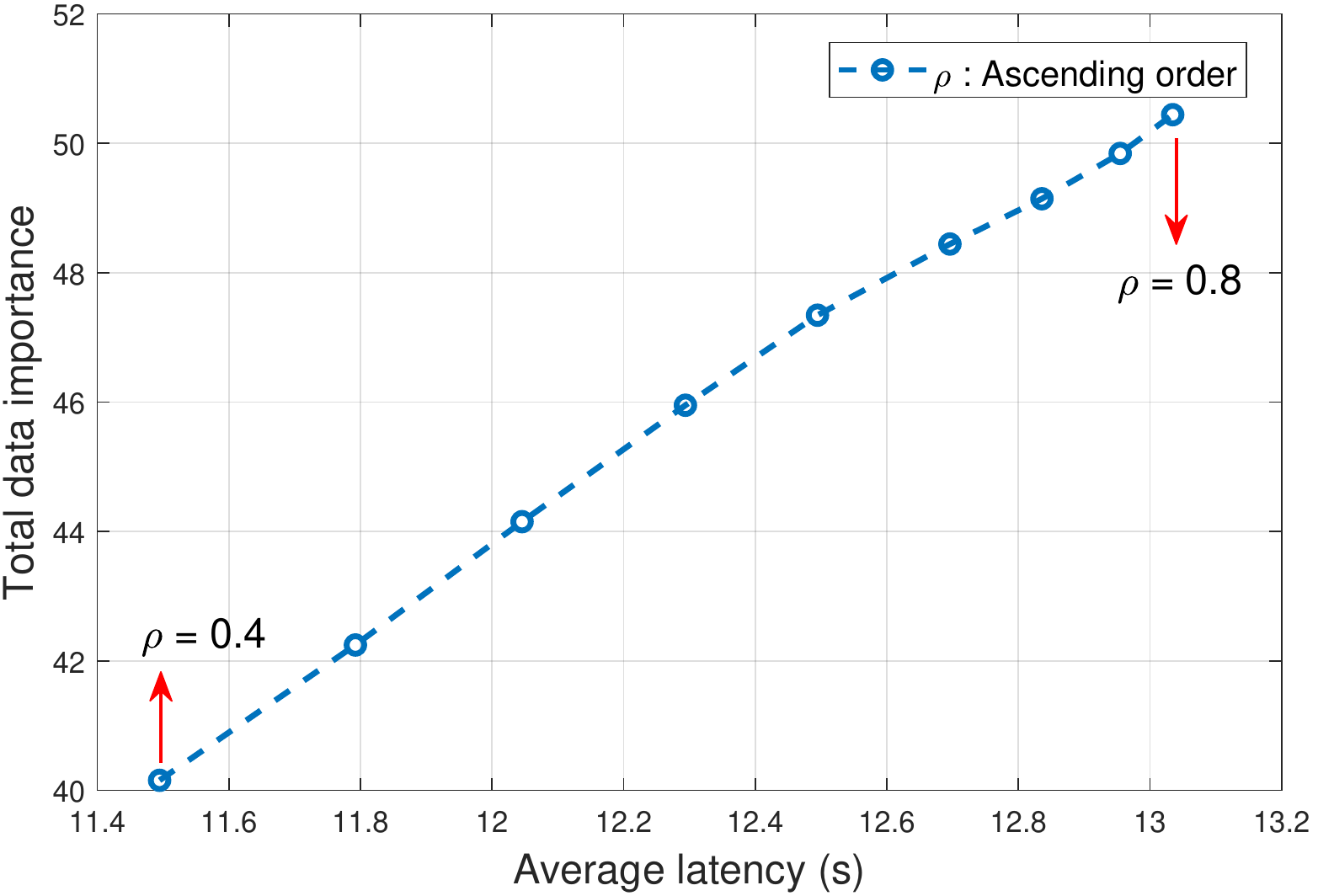}\\
  \caption{Tradeoff between data importance and delay.}\label{tradeoff}
\end{figure}
Fig. \ref{tradeoff} illustrates that a tradeoff exists between data importance and the transmission latency. The value of $\rho$ starts from $0.4$ to $0.8$ under the step of $0.05$. This figure shows that a large value of $\rho$ leads to higher data importance and longer transmission latency, and vise versa. Thus, the operators can select a suitable value of $\rho$ according to their specific requirements.

\begin{figure}
\center
\subfigure[Data importance and the number of selected edge nodes.]{
\begin{minipage}[b]{0.42\textwidth}
\includegraphics[width=\textwidth]{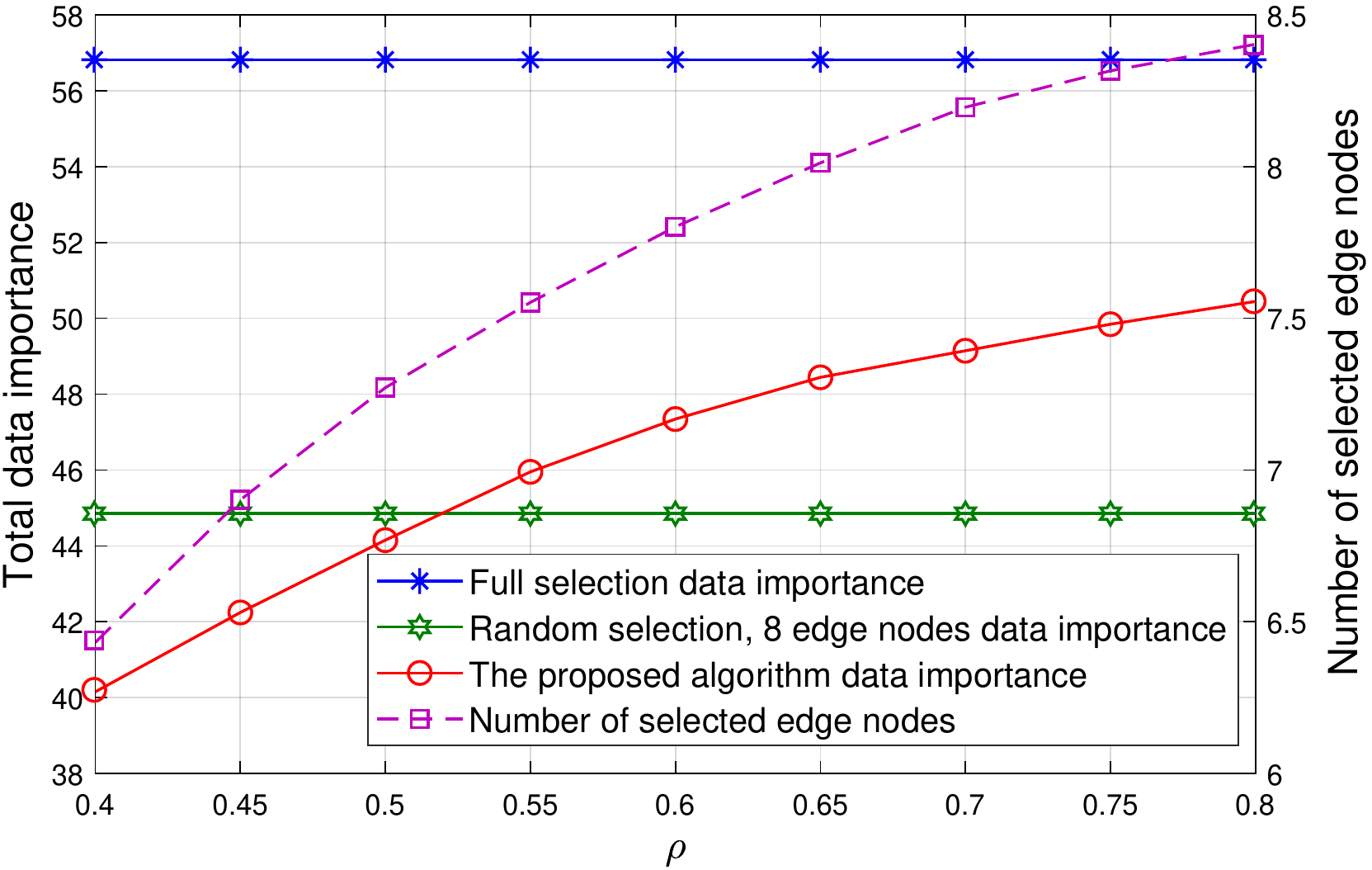}
\end{minipage}\label{conv_semi}
}
\subfigure[Latency.]{
\begin{minipage}[b]{0.42\textwidth}
\includegraphics[width=\textwidth]{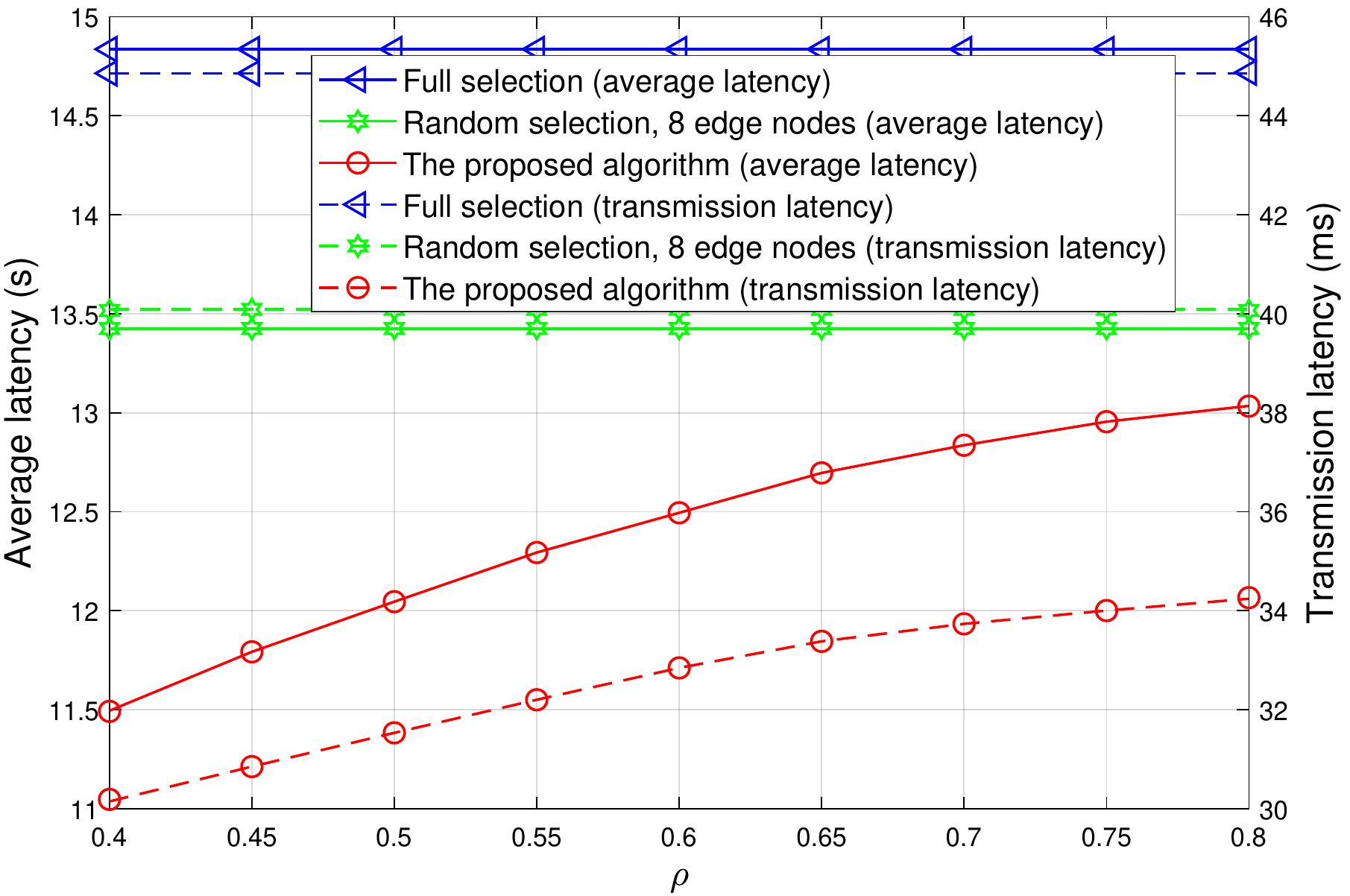}
\end{minipage}\label{conv_BCD}
}
\caption{Performance among the number of selected edge nodes, data importance, and latency under different mechanisms.}\label{performance}
\end{figure}
In Fig. \ref{performance}, we present the performance among the number of selected edge nodes, data importance, and latency under various weight factors. Fig. \ref{performance}(a) shows the number of selected edge nodes and total data importance in different weight factors under various algorithms. From this subfigure, the number of selected edge nodes increases with the weight factor $\rho$. When the value of weight factor $\rho$ is small, i.e., the associated edge nodes are small, the proposed algorithm has a lower data importance than the random selection scheme. With the increment of associated edge nodes, the circumstance changes, which has been explained in Fig. \ref{accuracy}.  However, the full selection scheme always has the highest value of data importance at the cost of higher latency, which is shown in Fig. \ref{performance}(b). Fig. \ref{performance}(b) shows the full selection scheme suffers the highest latency, and the proposed algorithm has the lowest latency after scheduling.
Intuitively, the transmission latency is much lower than the total latency, which means the data training time is huge. Moreover, the transmission latency may not meet the requirements of ultra low latency C-ITS devices. Under this circumstance, we can enlarge the wireless bandwidth by some resource management technologies.

\section{Conclusion}
This work proposes a novel SHFL framework that consists of local, edge, and cloud nodes to provide safety, traffic efficiency, and comfortable infotainment services for C-ITS. Specifically, homogeneous devices are allowed to associate with one edge node. Therefore, we adopt the synchronous aggregation model for edge nodes. On the contrary, for the heterogeneous edge aggregation models, we introduce a semi-asynchronous aggregation model for the cloud node, where parts of the fastest training edge models can be uploaded at each iteration. Moreover, we investigate an effective edge-cloud update method to keep the personalities of the edge nodes. For communication-efficiency, we propose a joint edge node association and resource allocation strategy, which illustrates a tradeoff between training accuracy and transmission latency. A distributed ADMM-BCU algorithm has been adopted to solve the proposed optimal MINLP problem. Numerical results show that our proposed scheme can accelerate the training process and improve the performance for C-ITS.

\end{document}